\documentclass{article}

\usepackage{graphicx}
\usepackage{latexsym}
\usepackage[pdftex]{hyperref}
\usepackage{subcaption}
\usepackage{amssymb,amsmath,amsthm}
\usepackage{apalike}

\newtheorem{defi}{Definition}
\newtheorem{prop}{Proposition}
\newtheorem{algo}{Algorithm}
\newtheorem{assum}{Assumption}
\newtheorem{coro}{Corollary}
\newtheorem{remark}{Remark}

\newtheorem{theo}{Theorem}

\title{Active Learning for Regression based on Wasserstein distance and GroupSort Neural Networks}
\author{Bobbia Benjamin, Picard Matthias}

\begin{document}
\maketitle

\begin{abstract}
This paper addresses a new active learning strategy for regression problems. The presented Wasserstein active regression model is based on the principles of distribution-matching to measure the representativeness of the labeled dataset. The Wasserstein distance is computed using GroupSort Neural Networks. The use of such networks provides theoretical foundations giving a way to quantify errors with explicit bounds for their size and depth. This solution is combined with another uncertainty-based approach that is more outlier-tolerant to complete the query strategy. Finally, this method is compared with other classical and recent solutions. The study empirically shows the pertinence of such a representativity-uncertainty approach, which provides good estimation all along the query procedure. Moreover, the Wasserstein active regression often achieves more precise estimations and tends to improve accuracy faster than other models.
\end{abstract}

\section{Introduction}
 Collecting data is a significant challenge in machine learning, and more generally in statistics. The amount of data necessary to get a sharp estimation of a function can get unreasonably high, especially when the task involves high-dimensional objects. In various applications, extracting and gathering enough unlabeled data is not a big challenge. However, labeling them can be a very costly and time-consuming process. In fields such as statistical physics, we would often need to run complex simulations or call an expert to label data manually. Hence we want to be able to get a satisfying estimation with a more compact set of labeled data. Among the proposed solutions, we can mention few-shot learning \cite{FFFP2006} and transfer learning \cite{BDBCKPV2010}. Those solutions leverage from another model previously trained on a similar task to simplify our model's training. To dodge the issue, we can also resort to  generative adversarial models \cite{GoodGAN2014} or diffusion models \cite{TDGS2023}: the idea is to augment a dataset using generated data, these approaches need to learn not only the targeted task but how to create new data and criticize them to prevent the addition of incoherent points into the pool. In this paper, we will take another approach called active learning. In this framework, we can call an ``oracle'' that has the ability  to label samples of our data with a certain cost. The keystone of this learning procedure is to find the most relevant samples that both maximize the performances of the estimator and minimize the query cost.
 
\subsection{Active Learning framework and notations}
We are interested in performing active learning for a regression task using pool-based sampling. We consider an unknown function $h$ from $\mathbb{R}^d$ to $\mathbb{R}$. We aim to estimate $h$ given a sample $(X_i,Y_i)_{i=1}^n$ of i.i.d copies of the random variables $(X,Y)\in \mathbb{R}^d\times\mathbb{R}$ with $Y=h(X)$. We call $P_X$ the unknown distribution of $X$. In our framework, not all values $Y_i$ are observed for every value $X_i$. The sample $(X_i,Y_i)_{i=1}^n$ can be divided into two subsets $K$ and $U$ (respectively for known and unknown values of $h(X)$) such that for all $X_i \in K$ the value of $h(X_i)=Y_i$ is observed but not for $X_i \in U$. The most natural setup which requires active learning is when $K$ (of size $n_K$) is significantly smaller than $U$ (of size $n_U$) and not large enough to allow a fairly good estimation of $h$. Finally, $\hat h$ denotes an estimator of $h$. In this paper, the estimator $\hat h$ is chosen to belong to one class of neural networks. However, the methodology developed here focuses mainly on the query step. Hence it can be adapted to other classes of estimators $\hat h$.

The error made by $\hat h$ is measured with a strictly convex loss function $l$ from $\mathbb{R}^2$ to $[0,+\infty)$. The estimator of  $h$ is chosen as a minimum of the expected error risk:
\begin{equation}
\mathcal{R}_{P_X}(\hat h)=\mathbb{E}\left(l(\hat h(X),Y))\right)
\end{equation}
 Obviously, since the probability measure $P_X$ is unknown, the explicit computation of $\mathcal{R}_{P_X}(\hat h)$ is hopeless. Hence, we are going to define $\hat h$ minimizing the empirical counterpart of the expected error risk:
\begin{equation}
\mathcal{R}_{\mathbb{P}_n}(\hat h)=\frac{1}{n}\sum_{i=1}^n l(\hat h(X_i),Y_i)
\end{equation}
where $\mathbb{P}_n$ denotes the empirical measure of $(X_1,\dots,X_n)$:
\[
\frac{1}{n}\sum_{i=1}^n\delta_{X_i}.
\]
Since most values $Y_i$ are unobserved, $\mathcal{R}_{\mathbb{P}_n}(\hat h)$ is uncomputable. Hence our purpose is to find out a batch $B\subset U$ with a fixed size $n_B$ providing the best improvement of the estimation of $h$ on $K\cup B$ compared to the estimation using only $K$. After choosing B, we would need to send it to the oracle to label it. In pool-based active learning, the idea is to accomplish this task by leveraging our pool of unlabelled data $U$. During the past years, several options to identify the best batch to query have been explored: the first method proposed by \cite{LG1994} consisted in querying the points with the highest prediction uncertainty (uncertainty-based sampling). Numerous derived strategies were investigated: targeting points with the highest expected error \cite{RM2001}, expected variance \cite{ZO2000}, or cross-entropy \cite{S2001} when selecting the batch. Other approaches like query-by-committee \cite{BRK2007} or greedy sampling \cite{WU2019} tried not only to focus on the reduction of the model error but also on the prospect of adding diversity in the choice of $B$ to get a pool more representative of the underlying distribution. We refer to the monography \cite{SBAL2012} and the survey \cite{RXYXPL2020} for overall information about active learning and details on query strategies. A more recent idea that has shown some good results is to directly create a representativeness-based model by selecting $B$ using probability distribution matching \cite{SZGW2020}. The aim is to label the samples that will minimize the distance between $P_X$ and $\mathbb{P}_{K \cup B}$. This means that we achieve a selection of the best empirical representation of $P_X$ among empirical measures of size $n_K+n_B$ (according to the chosen distance). Afterward, we send the batch $B$ to the oracle to get the value of $h(X_i)$ for $X_i \in B$ and learn $\hat h$ as the minimum of:
\begin{equation}
\mathcal{R}_{P_{K\cup B}}(\hat h)=\frac{1}{n_B+n_K}\sum_{x \in K\cup B} l(\hat h(x),Y_i)
\end{equation}

This process can be iterated as long as there is still enough unlabelled data and enough budget.

In the following, the difference between probability distributions is measured using the  Wasserstein distance, which has already proven efficient in classification tasks \cite{SZGW2020}. 

\subsection{ The Wasserstein distance}\label{subsubsec_was_def}
\begin{defi}
Let $P$ and $Q$ two probability measures on a metric space $(\mathcal{X},c)$. For $p\in[1,\infty[$, the Wasserstein distance of order $p$ between $P$ and $Q$ is defined as 
\begin{equation}
W_p(P,Q)=\inf\lbrace \mathbb{E}(c(X,Y)^p)^{1/p}, \ X\sim P, Y\sim Q\rbrace.
\end{equation}
\end{defi}
Note that the Wasserstein distance is a proper distance only on the subset of probability measures $P$ such that there exists a point $x_0\in \mathcal{X}$ where $\int_\mathcal{X} c(x,x_0)^p P(\mathrm{d}x)< \infty$. In our active learning framework, this is not an issue: since we are focusing on using the 1-Wasserstein distance, this condition only requires assuming that $P_X$ has a finite expectation. The first-order Wasserstein distance is of main interest since it can be expressed as a supremum over all 1-Lipschitz functions with respect to the distance $c$ using the following proposition:
\begin{prop}[Kantorovich-Rubinstein duality \cite{V09}]
Let $P$ and $Q$ two probability measures on the same space $\mathcal{X}$, and $c$ a cost function. The 1-order Wasserstein distance is defined as 
\begin{equation}
W_1(P,Q)=\underset{\varphi}{\sup}\int_\mathcal{X}\varphi \ \mathrm{d}P-\int_\mathcal{X}\varphi \ \mathrm{d}Q.
\end{equation}
Where the supremum is taken over all $1-$Lipschitz functions with respect to the cost $c$, namely $|\varphi(x)-\varphi(y)|\leq c(x,y)$ for all $x,y\in \mathcal{X}$.
\end{prop}
 The present formulation is a dual formula holding only for the 1-order Wasserstein distance. It is much more interpretable and easier to compute than the original definition. In the present work, we are only working with this formulation. Hence we can regard it as a definition.
 
Even if the proper value of a Wasserstein distance cannot be interpreted by itself, it characterizes the weak convergence of probability distributions (see e.g.\cite{V09}), which guaranty the convergence of our method. Moreover, the existence of large deviation bounds and generally good behavior with respect to Lipschitz functions may help to provide a rate of convergence \cite{FG2015} but we let this for future works. We can also refer to the monography \cite{PZ2020} for other useful examples of its use in statistics.

\section{Wasserstein Active Regression}\label{sec:war_description}

\subsection{Theoretical foundations}\label{subsec:theory}

Before introducing the proper model, some necessary assumptions about the mathematical background of the approach are stated. We first assume that the random variable $X$ takes its values in $\mathcal{X}$ which is a compact subset of $\mathbb{R}^d$ endowed with a norm $\|.\|$. This is rather intuitive and allows the use of numerous results about Lipschitz approximations. To prove the next theorem, we will also have to assume that the target function and the cost function are 1-Lipschitz respectively with respect to $\|.\|$ and the $L^1$ norm. These can seem like restrictive conditions, but we will see later that we can get rid of them in practice. The following theorem (closely related to \cite{SZGW2020}) acts as a guideline for the presented approach.
\begin{theo}\label{thm:emp_risk_bound_wass}
If the functions $\hat h$ and $G:x\in \mathcal{X} \rightarrow l(\hat h(x),h(x))$ are 1-Lipschitz functions, we have:
\begin{equation}
|\mathcal{R}_{P_X}(\hat h)-\mathcal{R}_{P_{K\cup B}}(\hat h)| \leq W_1\left(P_X,P_{K\cup B}\right)
\end{equation}
\end{theo}

\begin{proof}
By definition of $G$, we can write
\begin{align*}
|\mathcal{R}_{P_X}(\hat h)-\mathcal{R}_{P_{K\cup B}}(\hat h)|=&\left|\int_\mathcal{X}G\ \mathrm{d}P_X-\int_\mathcal{X}G\ \mathrm{d}P_{K\cup B}\right|\\
\leq & \underset{\varphi \in Lip_1}{\sup}\left|\int_\mathcal{X}\varphi \ \mathrm{d}P_X-\int_\mathcal{X}\varphi\ \mathrm{d}P_{K\cup B}\right|.
\end{align*}
The Kantorovich-Rubinstein duality directly leads to the result.

\end{proof}


 Working with Lipschitz functions offers several benefits. On the one hand, they are necessary to compute the Wasserstein distance (when using the Kantorovich-Rubinstein duality). On the other hand, they are robust and help us prevent overfitting on the estimation of $h$. More precisely, for  $\varepsilon >0$, if for $x,\ y \in \mathcal{X}$ such that $\|x-y\|\leq \varepsilon$ then $\| \hat h(x)-\hat h(y)\|\leq \varepsilon$ hence the robustness. The protection against overfitting might be less easy to see. Consider a point $y$ such that $|h(y)-\hat h(y)|\leq \delta$ for a given $\delta>0$. Then for all $x$ such that $\|x-y\|\leq \varepsilon$, the triangular inequality entails $|\hat h(x)- h(x)|\leq 2\varepsilon+\delta$.
As a consequence, even if an observation of $h(x)$ is very noisy, this noise can't affect much the estimation as long as there exists another point in a neighborhood of $x$ with a sharp estimation.
 
\begin{remark}
We can generalize the assumptions previously presented to estimate any Lipschitz functions with this model. Indeed, the cornerstone of theorem \ref{thm:emp_risk_bound_wass} is that $G$ is 1-Lipschitz. Hence, if $h$ is $\lambda_1$-Lipschitz, $\hat h$ is $\lambda_2$-Lipschitz and the loss function $l$ is $\lambda$-Lipschitz, then for $x,y\in \mathcal{X}$
\begin{align}
|G(x)-G(y)| = & |l(\hat h(x),h(x))-l(\hat h(y),h(y))| \\
         \leq & \lambda(\lambda(|h(x)-h(y)|+|\hat h(x)-\hat h(y)|))\\
         \leq & \lambda_1 \|x-y\|+\lambda_2 \|x-y\|\\
         =&\lambda(\lambda_1+\lambda_2)\|x-y\|.
\end{align}
It is enough to divide $G$ by $\lambda_1+\lambda_2$ to get a 1-Lipschitz function. Actually, we can completely remove this factor because the process of minimizing the cost function will not change, with or without it. So in practice, it is not taken into account, the only constraint is that $\hat h$ and $h$ should be Lipschitz. This is not restrictive since we assume $\mathcal{X}$ to be compact. Indeed, any piecewise differentiable functions on compact sets are Lipschitz.
\end{remark}

\begin{coro}\label{coro:wass_emp_dist}
Assume that $\mathbb{E}(X^{d/2+2})<\infty$, then for any batch $B \subset U$, we have:
\begin{equation}\label{eq:coro_wed}
|\mathcal{R}_{P_X}(\hat h)-\mathcal{R}_{\mathbb{P}_{K\cup B}}(\hat h)| \leq W_1\left(\mathbb{P}_n,\mathbb{P}_{K\cup B}\right)+O_{\mathbb{P}}\left(\frac{1}{n^{d/2}}\right).
\end{equation}
\end{coro}

\begin{proof}
The triangular inequality gives 
\[
W_1(P_X,\mathbb{P}_{K\cup B})\leq \ W_1(P_X,\mathbb{P}_{n})+W_1(\mathbb{P}_n,\mathbb{P}_{K\cup B}).
\]
For any $\varepsilon\in (0,1)$, Theorem 2 from  \cite{FG2015} applied to $W_1(\mathbb{P}_X,\mathbb{P}_n)$ provides the existence of positive constants $\alpha$ and $\beta$ not depending on $n$ such that 
\[
\mathbb{P}\left(W_1(\mathbb{P}_X,\mathbb{P}_n)>\varepsilon\right)\leq \alpha e^{-\beta n \varepsilon^d}+\alpha \frac{\varepsilon^{1-d/2}}{n^{d/2}}.
\]
Hence the result.
\end{proof}

\begin{remark}
The assumption $\mathbb{E}(X^{d/2+2})<\infty$ can be relaxed, in fact result from \cite{D68} allows to show that convergence remains true but with a speed of convergence in $n^{-1/d}$. Which is much worst especially in the case of high dimensional covariates.
\end{remark}


This bound, based on the result from \cite{FG2015}, suggests that making $W_1\left(\mathbb{P}_n,\mathbb{P}_{K\cup B}\right)$ decrease may reduce the gap between $\mathcal{R}_{\mathbb{P}_{K\cup B}}(\hat h)$ and $\mathcal{R}_{\mathbb{P}_{X}}(\hat h)$. Another advantage of this formulation is that $W_1\left(\mathbb{P}_n,\mathbb{P}_{K\cup B}\right)$ is computable whereas $W_1\left(\mathbb{P}_X,\mathbb{P}_{K\cup B}\right)$ is not.

\begin{remark}
    Note that the distance $W_1\left(\mathbb{P}_n,\mathbb{P}_{K\cup B}\right)$ may be small with respect to $O_{\mathbb{P}}\left(\frac{1}{n^{d/2}}\right)$, especially when the subset $K$ is large with respect to $U$. However, this is not the range that is focused on this active learning procedure. In addition, the term $O_{\mathbb{P}}\left(\frac{1}{n^{d/2}}\right)$ does not depend on $B$, as a consequence it can be interpreted as the remaining error bound after all the dataset has been queried.
\end{remark}

We can rewrite the corollary \ref{coro:wass_emp_dist} to identify the milestones of our active learning strategy:
\begin{equation}\label{eq:bound_milestone}
    \mathcal{R}_{P_X}(\hat h)\leq \mathcal{R}_{\mathbb{P}_{K\cup B}}(\hat h) + W_1\left(\mathbb{P}_n,\mathbb{P}_{K\cup B}\right)+c_n,
\end{equation}
with $c_n$ not depending on $B$. Hence the minimization of $\mathcal{R}_{P_X}(\hat h)$ can be achieved by optimizing the right-hand side of inequality \eqref{eq:bound_milestone} on $\hat h$ and $B$. This leads to the construction of two criteria for the batch selection. More precisely, a score function $S$ is considered, and points $x\in U$ maximizing this score are added to the batch. In the active learning framework, such a function is often called an acquisition function. The bound \eqref{eq:bound_milestone} leads to the consideration of a score function of the form:
\begin{equation}
    S(x)=r(x)+w(x),\quad x\in U, 
\end{equation}
were $r$ relies on the minimization of $\mathcal{R}_{\mathbb{P}_{K\cup B}}(\hat h)$ whereas $w$ relies on the Wasserstein distance. Note that there is no counterpart to the constant $c_n$ in the score function since this constant does not depend on $B$. The construction of the functions $r$ and $w$ are detailed in the following sections. 

\subsection{Training the estimator}

To achieve the learning step, we minimize the term $\mathcal{R}_{P_K}(\hat h)$ defined as: 
\begin{equation}
\mathcal{R}_{P_K}(\hat h)=\frac{1}{n_K}\sum_{ x \in K} l(\hat h(x),h(x))\label{eq:coro2}
\end{equation}
with respect to $\hat h$. Afterward, we will query data by using a hybrid strategy based on two criteria: we want to pick the batch $B$ that makes $W_1\left(\mathbb{P}_n,\mathbb{P}_{K\cup B}\right)$ decreases the most, and we  also want to maximize the uncertainty of $\hat h$  in order to encourage our model to train on data where $\hat h$ struggles at estimating the right labels. Intuitively, the batch $B$ has to be composed of points containing as much information as possible about the underlying data distribution and located in places where $\hat h$ can still improve its estimations.

\subsection{Minimizing the Wasserstein distance}\label{sec:war_imple}

Considering the first criterion, it is clearly unreasonable to compute all Wasserstein distances $W_1\left(\mathbb{P}_n,\mathbb{P}_{K\cup B}\right)$ for each possible $B$. To tackle this issue we compute the Wasserstein distance using the Kantorovitch-Rubinstein duality between $\mathbb{P}_n$ and $\mathbb{P}_{K}$ by maximizing $\hat W$ defined as   
\begin{align}
\hat W(\varphi)&=\left|\int_\mathcal{X}\varphi \ \mathrm{d}\mathbb{P}_n-\int_\mathcal{X}\varphi \ \mathrm{d}\mathbb{P}_K\right| \\
   &=\left|\frac{1}{n}\sum_{i=1}^n\varphi(X_i)-\frac{1}{n_K}\sum_{ x\in K}\varphi(x)\right|\label{eq:hatW_emp}
\end{align}
with $\varphi$ belonging to the set of 1-Lipschitz functions. Note that the maximisation of $\hat W$ allows to approximate the Wassserstein distance since $W_1(\mathbb{P}_n,\mathbb{P}_{K})=\sup\{ \hat W(\varphi)\ | \ \varphi \text{ is 1-Lipschitz}\}$. From a practical point of view, $\varphi$ is a neural network that we optimize by using $\hat W$ as a cost function. The function $\varphi$ hence obtained entirely characterizes $W_1\left(\mathbb{P}_n,\mathbb{P}_{K}\right)$. To minimize $W_1\left(\mathbb{P}_n,\mathbb{P}_{K}\right)$, we have to choose the batch $B$ belonging to $U$ that maximizes $\varphi$, and add it to $K$ by labeling it.



 We can deduce from Equation \eqref{eq:hatW_emp} that the maximization of $\hat W(\varphi)$ involves the minimization of $\varphi$ near points in $K$, and its maximization near points in $U$. Because $\varphi$ must be 1-Lipschitz, the fluctuations of the function are constrained. An intuitive consequence is that $\varphi$ has to make compromises: some points must be "ignored" in order to be maximum near more relevant locations. Indeed,  $\varphi$  tends to take its highest values (respectively, lowest) near clusters of unlabelled points (respectively, labeled points). However, if there are no labeled points in the neighborhood, $\varphi$ can increase without restrictions. These behaviors are illustrated in the figure below: 

\begin{figure}
\centerline{\includegraphics[width=1\linewidth, height=5cm]{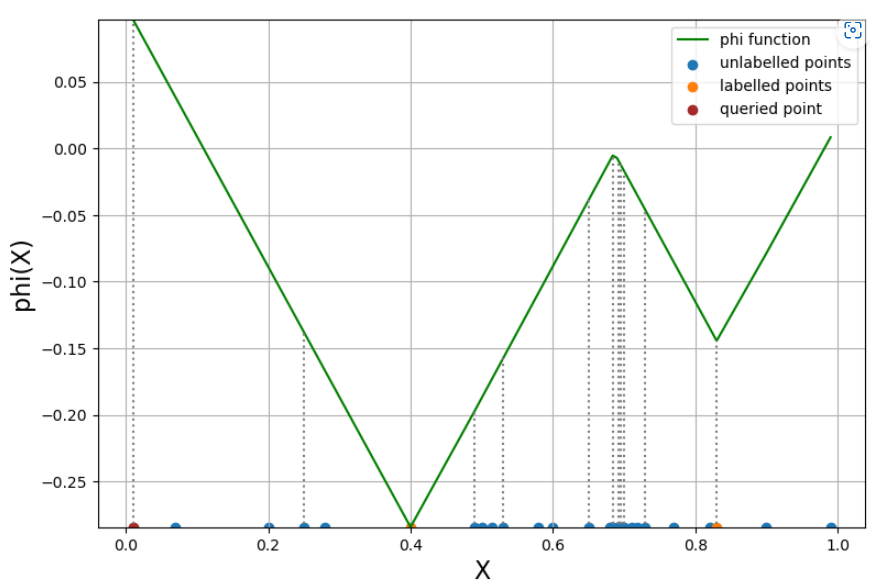}}
\caption{Values of $\varphi$ on a labelled and an unlabelled distribution in $[0,1]$}\label{courb:Wasserstein}
\end{figure}

Here, $\varphi$ is maximized for unlabelled points that are very far from any labeled points, or packed in a cluster. Consequently, in the context of active learning, we can use the Wasserstein distance as a diversity-based and representativity-based sampling method.

 In order to guarantee a proper estimation of the Wasserstein distance, we have to impose $\varphi$ to be a 1-Lipschitz function. A classic method would be to penalize the gradient of $\varphi$ for enforcing its boundedness \cite{FADC2017}, \cite{GFPC2021}. But, in this case, we may underestimate the Wasserstein distance in a proportion that we can't quantify (to our knowledge). Consequently, we chose to use Group Sort Neural Networks, which present good properties for the estimation of $1-$Lipschitz functions (see section \ref{sec:GS_NN} for the definition and mains properties). A density result from \cite{ALG2019} ensures that the supremum over such networks is equal to the Wasserstein distance (see Proposition \ref{prop:was_GS_NN}).

\subsection{Minimizing the predictions uncertainty and query procedure}\label{subsec:WAR_implementation}

 This section addresses the problem of the construction of the score function $S$ based on the Wasserstein distance with an uncertainty-based method. To query the most relevant data in $U$, we need to evaluate each of them through the function $S$ and build $B$ by taking the points that maximize it. As presented in section \ref{subsec:theory}, the score function includes two terms $r$ and $w$. For $x\in U$, the term $w(x)$ relies on the decay of the Wassertein distance $W_1(\mathbb{P}_n,\mathbb{P}_{K\cup B})$ when $x$ is added to $B$. In this sense, $w$ is chosen proportional to $\varphi$ maximizing \eqref{eq:hatW_emp}, namely
 \begin{equation}
     w(x)= \beta \varphi(x)
 \end{equation}
with $\beta$ a positive constant.
The choice of $r$ is not so straightforward. In the best-case scenario, we would be able to retrieve the loss and make $r(x)=l(\hat h (x), h(x))$ with $x\in U$ to query the points where this loss is the largest (as hinted by \eqref{eq:coro2}). However, for such $x$, the variable $h(x)$ is not observed. To our knowledge, there is no efficient way to estimate the loss without losing too much generality. This issue comes from the fact that a regression framework with no assumptions on the second-order derivative of $h$ leads to a non-convex optimization problem. So instead of focusing on the loss function, a disagreement sampling approach will be considered to reduce the uncertainty of the predictions. The objective is to determine the regions where the model struggles the most to find a stable estimation of the output. To do so, several copies of $\hat h$ are trained. Then, data in $U$ are ranked according to the estimated standard deviation $s_h$ of their prediction. Formally, the uncertainty $s_h(x)$ of a point $x\in U$ is defined as: 
\begin{equation}
    s_h(x)=\sqrt{\frac{1}{k}\sum_{i=1}^k\left( \hat h_i(x)-\frac{1}{k}\sum_{i=1}^k \hat h_i(x)\right)^2}
\end{equation}
where $(\hat h_1,\dots,\hat h_k)$ are retrained version of $\hat h$. The points picked out are the ones with the highest uncertainty.
However, a main concern about uncertainty-based sampling is its tendency to give more importance to unrepresentative regions of the input space, such as outliers. To correct this undesired behavior, for each point in $U$, we penalize its estimated standard deviation prediction by the mean of its distance from the other points. 
Moreover, a hyperparameter $\alpha > 0$ is added to weight the penalization of points located far from the barycenter of data distribution. Note that this hyperparameter might be misleading for distributions with several data clusters. So without much knowledge of the data distribution, it might be safer to put the hyperparameter to relatively low values (less than $1$) to reduce the risk of creating a bias. Gathering these considerations lead to:
 \begin{equation}
r(x)= \frac{s_h(x)}{\left(\frac{1}{n}\sum_{i=1}^n\| X_i-x\|_2\right)^\alpha},\quad x\in U.     
 \end{equation}
Note that $r$ depends on the values taken by our pool of estimators $(\hat h_1,\dots,\hat h_k)$, whereas $w$ depends only on the probability distribution of the covariates. Such dependencies can induce a major difference between the scale of the two terms $r$ and $w$. In order to tackle this issue those terms are divided  by their respective empirical standard deviation on $U$. Such homogenization can avoid the situation where one criterion becomes the sole discriminator in our score function because of its bigger variance. Nonetheless, this gives the final score function used to evaluate every $x\in U$  to measure their benefit in order to minimize $ \mathcal{R}_{P_X}(\hat h)$:
  \begin{equation}\label{eq:finalscore}
  S(x)=\frac{s_h(x)}{s_r \left(\frac{1}{n}\sum_{i=1}^n\| X_i-x\|_2\right)^\alpha}+\beta\frac{\varphi(x)}{s_\varphi}
 \end{equation}
where $s_r$ and $s_\varphi$ are respectively the empirical standard deviation of $r$ and $\varphi$.
Finally, the queried batch contains the points for which the equation \eqref{eq:finalscore} reaches the maximum value. The hyperparameter $\beta$ can be tuned by the user to decide whether to give more weight to one criterion or to the other.

The score function is based on three of the main strategies used in active learning: diversity, representativity, and uncertainty. This allows our method to leverage several important aspects of the data distribution and of the chosen estimator. The two hyperparameters $\alpha$ and $\beta$ give some flexibility to tune the function. We detail the algorithm that was created to apply this method:

\begin{algo}[One iteration of batch selection]\textbf{ }\label{algo}
\begin{itemize}
    \item \textbf{arguments : }
    \begin{itemize}
    \item The labeled set $K$ and the unlabeled set $U$ .
    \item The batch size $n_B$.
    \item The vectors $\theta_h$ and $\theta_\varphi$ of parameters of the networks $\hat h$ and $\varphi$.
    We initialize those vectors randomly.
    \item $\eta_h$ and $\eta_\varphi$, the learning rates of $\hat h$ and $\varphi$ respectively.
    \item Uncertainty/representativity-diversity trade-off parameter $\beta$.
    \item Outlier penalization parameter $\alpha$.
    \end{itemize}
    \item \textbf{training $\hat h$ :} \\
         Update $\theta_h=\theta_h-\eta_h \nabla_{\theta_h}$ on $K$. The gradient descent is performed with the Adam algorithm \cite{KB2014}.
    \item \textbf{Query : }\\
    \quad Set $B$ empty.\\
    \quad While $size(B)\neq n_B$ do 
    \begin{enumerate}
        \item update $\theta_\varphi=\theta_\varphi+\eta_\varphi \nabla_{\theta_\varphi}$ on $K\cup B$. The gradient descent is performed with the Adam algorithm \cite{KB2014}.
        \item Compute $S(x)$ for all $x$ in $U$ 
        \item Add the $x$ with the highest $S(x)$ value to $B$ and remove it from $U$
    \end{enumerate}
    End While
    \item \textbf{Updating sets :} \newline $K=K\cup B$ and $U=U\backslash B $.
    
\end{itemize}
\end{algo}

\begin{remark}
In this algorithm, it is important to note that we retrain $\varphi$ to select each point of $B$ one by one. We do so in order to favor diversity in the queried batch. Indeed, if the same function $\varphi$ is used to pick out the entire batch $B$, the algorithm has a tendency to choose a restricted cluster of points. Which is counter-effective in regards to the objective of the coefficient $\varphi$ in the acquisition function $S$.
\end{remark}

To create the initial pool of labeled data, we will use the K-means algorithm: we define $b_0$ clusters and take the closest points to each cluster center.

\begin{algo}[Full WAR Algorithm]\textbf{}
    \begin{itemize}
        \item \textbf{Arguments} : 
        \begin{itemize} 
        \item Same as algorithm \ref{algo} 
        \item  The number of rounds $N_{query}$ (i.e. the number of call to algorithm \ref{algo}).
        \item Number of initial data $b_0$.
        \end{itemize}
        \item \textbf{Initialisation} :\newline
        Query the $b_0$ initial data using K-means
        \item \textbf{For} $N=1\dots N_{query}$ \textbf{Do} algorithm \ref{algo}.
    \end{itemize}
\end{algo}

\section{ Group Sort Neural Networks}\label{sec:GS_NN}

Informally, GroupSort neural networks introduced by \cite{ALG2019} are simply neural networks with an activation function that sorts the input of each layer. This activation function divides the input into several blocs of the same size and sorts the elements in each block in decreasing order. The output is hence the concatenation of the sorted blocs. As an example, if the input is 
\[
(9,6,10,8,6,10,7,9,5,9,6,4,5,8,8),
\]
the output of the GroupSort activation function with grouping size $3$ is 
\[
\begin{array}{cccccc}
(\underbrace{9,6,10},& \underbrace{8,6,10},& \underbrace{7,9,5},& \underbrace{9,6,4},& \underbrace{5,8,8})& \\
(10,9,6,& 10,8,6,& 9,7,5,& 9,6,4,& 8,8,5)& \\       
\end{array}
\]
The choice of neural networks using the GroupSort activation function is mainly motivated by the fact that \cite{ALG2019} has shown the density of the set of such neural networks in the set of 1-Lipschitz functions. Which allows giving an exact expression of the Wasserstein distance in terms of such neural networks (see Proposition \ref{prop:was_GS_NN}). They also showed empirically that they converge faster than other methods when it comes to estimating $1-$Lipschitz functions. Recent works \cite{BST2022,BCST2020,STCKB2022} highlight the pertinence of GroupSort neural networks in this area, hence their importance for query by probability matching applications using the Wasserstein distance.

\subsection{Presentation}

\begin{defi}[Group Sort activation function \cite{ALG2019}]
Let $k$ and $n \in \mathbb{N}$, the Group Sort function $\sigma_k$ is a function from $\mathbb{R}^n$ to $\mathbb{R}^n$ which split an input $(x_1,\dots,x_n)$ into $n/k$ blocs of size $k$ and return the vector in which each bloc is ordered. More formally, one define for $i=1,\dots,n/k$, $G_i=\lbrace x_{k(i-1)},\dots,x_{k i}\rbrace$ and 
\begin{equation}
\sigma(x_1,\dots,x_n)=(\tilde G_1,\dots,\tilde G_{n/k})
\end{equation}
where $\tilde G_i$ is the ordered version of $G_i$. That is, for $x_k$ and $x_j \in \tilde G_i$, $x_k\leq x_j$ if $k<j$. The integer $k$ is called the grouping size.
\end{defi}

\begin{defi}[GroupSort neural network \cite{TSB2021} ]
Let $q\in \mathbb{N}$ and $\alpha\in \Lambda$, with $\Lambda=\lbrace V_1,\dots,V_q,c_1,\dots,c_q\rbrace$ where the $V_i$ are matrices and $c_i$ are vector with dimension detailed in definition \ref{def:Lambda}. We define the feedforward neural network with $q$ layers $h_{\alpha}$, from $\mathbb{R}^d$ to $\mathbb{R}^p$, iteratively as
\begin{itemize}
    \item $l_1=\sigma(V_1 x + c_1)$
    \item $l_i=\sigma(V_i l_{i-1} +c_i)$ for $i=2\dots q$
\end{itemize}
where the function $\sigma$ is a Group Sort activation function. 
\end{defi}

The assumption characterizing the networks of interest involves two matrix norms defined for a matrix $A$ and a vector $X$ as  
\[
\|A\|_\infty= \underset{\|X\|_\infty=1}{\sup}\|AX\|_{\infty} \quad \text{and}\quad \|A\|_{2,\infty}=\underset{\|X\|_2=1}{\sup}\|AX\|_{\infty}
\]
where $\|.\|_2$ denotes the euclidean norm whereas $\|.\|_\infty$ denotes the norm of the supremum.
\begin{assum}[\cite{TSB2021}]\label{assum:groupsort}
For $\alpha\in \Lambda$, we assume that 
\begin{itemize}
    \item $\|V_1\|_{2,\infty}\leq 1$
    \item $\max(\|V_2\|_\infty,\dots,\|V_q\|_\infty) \leq 1$
    \item there exists $K>0$ such that $\max(\|c_1\|_\infty,\dots,\|C_q\|_\infty) \leq K$.
\end{itemize}
For $k \geq 2$, we call $GS_k$ the set of GroupSort neural network with grouping size $k$ satisfying this assumption.
\end{assum}
In practice, the construction of a neural network fulfilling assumption \ref{assum:groupsort} relies on Bjork's orthonormalization algorithm \cite{BB1971}. Finally, density results in the set of $1-$Lpischitz functions stated in \cite{ALG2019} highlight the fact that the Wasserstein distance can be computed exactly by only looking at the space of GroupSort neural networks.  
\begin{prop}\label{prop:was_GS_NN}
Let $\mathcal{X}$ a compact subset of $\mathbb{R}^d$ endowed with the $L_1$ metric. Consider $GS_2$ the set of all Group Sort neural networks with $\sigma_2$ as the activation function. Then, for $P$ and $Q$ two probability measures on $\mathcal{X}$, we have: 
\begin{equation}
W_1(P,Q)=\underset{\varphi \in GS_2}{\sup}\int_\mathcal{X}\varphi \ \mathrm{d}P-\int_\mathcal{X} \varphi \ \mathrm{d}Q.
\end{equation}
\end{prop}

\begin{proof}
By the Kantorovich-Rubinstein duality [11] we know that 
\[
W_1(P,Q)=\underset{\varphi \in Lip_1}{\sup}\int_\mathcal{X}\varphi \ \mathrm{d}P-\int_\mathcal{X}\varphi \ \mathrm{d}Q.
\]
Moreover Anil et al. (2019) have shown that the set $GS_2$ is constituted of 1-Lipschitz functions. And furthermore showed that this set is dense in the set of 1-Lipschitz function for the $L_\infty$ norm. So we conclude the proof establishing the continuity of the map 
\[
\begin{array}{cccl}
\Phi:& (Lip_1, \|.\|_\infty) & \rightarrow  & (\mathbb{R},|.|) \\
&\varphi &\mapsto & \int_\mathcal{X}\varphi \ \mathrm{d}P-\int_\mathcal{X}\varphi \ \mathrm{d}Q.
\end{array}
\]
In fact, we can show that this map is Lipschitz continuous. Indeed, let $\varphi_1$ and $\varphi_2 \in Lip_1$, we have 
\begin{align*}
    |\Phi(\varphi_1)-\Phi(\varphi_2)|= & \left|\int_\mathcal{X}\varphi_1 \ \mathrm{d}P-\int_\mathcal{X}\varphi_1 \ \mathrm{d}Q-\int_\mathcal{X}\varphi_2 \ \mathrm{d}P+\int_\mathcal{X}\varphi_2 \ \mathrm{d}Q\right| \\
    \leq &  \int_\mathcal{X}\left\|\varphi_1-\varphi_2\right\| \ \mathrm{d}P+\int_\mathcal{X}\left\|\varphi_1-\varphi_2\right\| \ \mathrm{d}Q \\
    \leq & 2\left\|\varphi_1-\varphi_2\right\|_\infty.
\end{align*}
\end{proof}

A sequence of Group Sort neural networks allowing the computation of the Wasserstein distance may involves networks with growing depth and size. However \cite{TSB2021} have provided, for a given error, bounds for both depth and size (see next section).
 
\subsection{Network Architecture}

This section present a result about sufficient depth and size for GroupSort neural networks in order to achieve a given precision on the estimation of the Wasserstein distance. \\

Since the set $GS_2$ contains all GroupSort neural networks, the density result involves the consideration of a neural network with non-bounded size and depth. Consequently, if $\varphi$ is restricted to some fixed size and depth, we will not be able to make a perfect approximation of $W_1(P,Q)$. Nevertheless, Tanielian et al. \cite{TSB2021} gave bounds for both depth and size considering a given error. This can help to create the right architecture for the network $\varphi$.

\begin{prop}\label{prop:error_was_GS}
Let $\varepsilon >0$, there exist $\varphi^*\in GS_2$ such that 
\begin{equation}
\left|W_1(P,Q)-\left(\int_\mathcal{X}\varphi^* \ \mathrm{d}P-\int_\mathcal{X}\varphi^* \ \mathrm{d}Q\right)\right|\leq \varepsilon,
\end{equation}
and, for any $0 < \delta < \varepsilon $ the dimensions of $\varphi^*$ is bounded as follows : 
\begin{itemize}
\item Size of $\varphi^*$ is a $O\left(\left(\frac{\sqrt{d}}{\varepsilon-\delta}\right)^{d^2}\right)$.
\item Depth of $\varphi^*$ is a $O\left(d^2\log_2\left(\frac{4\sqrt{d}}{\varepsilon-\delta}\right)\right)$.
\end{itemize}
\end{prop}

\begin{proof}[Proof of Proposition 3] 
 Let $ \varphi^* \in GS_2$, $P, Q$ two probability measures on $\mathcal{X}$ and $\delta >0$, we have
\begin{align*}
    \left|W_1(P,Q)-\left(\int_\mathcal{X}\varphi^* \ \mathrm{d}P-\int_\mathcal{X}\varphi^*\mathrm{d}Q\right)\right|&\leq\left|W_1(P,Q)-\left(\int_\mathcal{X}\varphi_0 \ \mathrm{d}P-\int_\mathcal{X}\varphi_0\mathrm{d}Q\right)\right| \\
    +& \left|\left(\int_\mathcal{X}\varphi_0 \ \mathrm{d}P-\int_\mathcal{X}\varphi_0\mathrm{d}Q-\int_\mathcal{X}\varphi^* \ \mathrm{d}P-\int_\mathcal{X}\varphi^*\mathrm{d}Q\right)\right|.
\end{align*}

Where $\varphi_0$ is a 1-Lipschitz function chosen such that, 
\[
\left|W_1(P,Q)-\left(\int_\mathcal{X}\varphi_0 \ \mathrm{d}P-\int_\mathcal{X}\varphi_0\mathrm{d}Q\right)\right|\leq \delta.
\]
Such a function exists since the Wasserstein distance is a supremum over 1-Lipschitz functions. As a consequence, this remark together with the last bound given in the proof of Proposition 2, we have for any $\delta$, that there exists a Lipschitz function $\varphi_0$ such that 
\[
\left|W_1(P,Q)-\left(\int_\mathcal{X}\varphi^* \ \mathrm{d}P-\int_\mathcal{X}\varphi^*\mathrm{d}Q\right)\right| \leq \delta+2\|\varphi_0-\varphi^*\|_\infty.
\]
Now, to estimate $W_1(P,Q)$ using neural network $\varphi^*$ in $GS_2$ accepting an error of $\varepsilon$, we can choose $\varphi^*$ such that 
\[
\|\varphi_0-\varphi^*\|_\infty \leq \frac{\varepsilon-\delta}{2}.
\]

For such an approximation, Tanielian et al. [29]  have shown that the depth of $\varphi^*$ can be chosen equal to $((\sqrt{d}/(\varepsilon-\delta))^{d/2})$ and its size as a $O(d^2\log_2(4\sqrt{d}/(\varepsilon-\delta)))$. Note that these two quantities are as greater as $\varepsilon$ and $\delta$ are closer. But, those bounds do not involve the values of $\varphi_0$, that's why we can choose $\delta$ smaller as we want. 
\end{proof}

The bounds provided in this section are huge, especially for high dimensional covariates. One can see in section \ref{sec:num_exp} that we can achieve the wanted accuracy with smaller networks.

\section{Numerical experiments}\label{sec:num_exp}

 For every dataset used, we split the trainset and the testset using a ratio of 80/20. We started the active learning procedure shown in Algorithm \ref{algo} initializing $K$ with around $2\%$ to  of the data available and we queried about $2\%$ of the trainset at each iteration. The models were trained by minimizing the mean squared error. The accuracy of the fitted models $\hat h$ was measured using RMSE. We retrieved the accuracy achieved in each model  when $25\%$ of the dataset were queried. Finally, we computed the area under the learning curve using the composite trapezoidal rule.

\subsection{Models and datasets used}\label{subsec:modat}

We compared the WAR model with 8 models, each using different query strategies: 
Random query, query by disagreement sampling with a committee of 5 neural networks, greedy Sampling on the Inputs (GSx) \cite{WU2019}, improved Greedy Sampling  (iGS) \cite{WU2019}, query by information density (euclidean method), query by information density (cosine method), improved Representativeness-Diversity Maximization (iRDM) \cite{Liu2021} and inverse-Distance based Exploration for Active Learning (IDEAL)\cite{Bem2023}.

 These tests were performed using five UCI datasets: Boston Housing \cite{HR1978}, Airfoil Self-Noise \cite{BPM1989}, Energy efficiency \cite{TX2012}, Concrete Slump Test \cite{Y2007}, and Yacht Hydrodynamics \cite{GOV1981}. In the presented tables these dataset are denoted respectively as "Bo", "Ai", "En", "Ya" and "Co". More information is available in the supplementary materials. There were several target columns in the Energy Efficiency and Concrete Slump Test datasets, we kept the Heating Load and SLUMP columns respectively, and dropped the others. Every categorical feature was already in a numerical form. We used a Min-Max scaler to scale the input data.
 
\subsection{Implementation}\label{subsec:implement}

In every model, the estimator $\hat h$ is a neural network with two hidden layers with respective sizes 16 and 32, completed by RELU activation functions. They all ran for 100 epochs in fullbatch. We add $L^2$ regularization on $h$ with a weight decay of 0.001 to prevent overfitting. Optimization was performed with an initial learning rate of 0.001 using Adam ($\beta_1$=0.9, $\beta_2$=0.999). Importantly, $\hat h$ parameters were not reset after each round in order to give more importance to the points queried during the first rounds (more on that in section \ref{sec:curriculum}). The hyperparameters of WAR and IDEAL were tuned for each dataset using Grid Search.

Specifically for the WAR implementation, we also had to train $\varphi$. We used Adam (still with $\beta_1$=0.9, $\beta_2$=0.999) and an initial learning rate of 0.01. $\varphi$ belongs to $GS_2$ (grouping size of 2). We used a committee of 5 estimators to compute the uncertainty. The results presented below are the RMSE averaged over 5 repetitions.

\subsection{Results}
The error made by the different algorithms after querying $25\%$ of the dataset is presented in table \ref{tab:rmse25}. The table \ref{tab:area} presents the area under the learning curves which are displayed in figure \ref{courb:concrete_slump}. Only the learning curves of three of the five datasets are presented, the two others are postponed to supplementary material. 

\begin{table}[h]
\caption{RMSE when 25\% percent of the data is labeled }\label{tab:rmse25}
\begin{center}
\begin{tabular}{|c|ccccc|}
\hline
  & Bo & Ai & En & Ya & Co \\ \hline
Random & 4.33 & 9.66 & 2.69 & 3.95 & 6.58 \\ \hline
Disagreement & 4.45 & 11.47 & 2.83 & 4.53 & 6.87 \\ \hline
GSx & 4.30 & 16.26 & 2.60 & 2.17 & 7.33 \\ \hline
iGS & 4.12 & 9.06 & \textbf{2.57} & \textbf{2.10} & 7.00 \\ \hline
Euclidean & 6.03 & 19.68 & 3.41 & 4.22 & 8.08 \\ \hline
Cosine & 5.60 & 19.57 & 3.83 & 3.25 & 8.32 \\ \hline
iRDM & 4.28 & 10.21 & 2.67 & 4.33 & 7.35 \\ \hline
IDEAL & 4.54 & \textbf{7.53} & 2.62 & 3.36 & 7.50 \\ \hline
WAR & \textbf{3.63} & 8.67 & 2.65 & 2.71 & \textbf{6.15} \\ \hline
\end{tabular}
\end{center}
\end{table}

The WAR algorithm (in light green) is consistently demonstrating a relatively fast convergence, often having the lowest RMSE for a given number of data. It also tends to outperform the other methods eventually. The results can be interpreted as a consequence of the effort to mitigate the risks of querying some outliers during the early stages of the process, thanks to the penalization $\alpha$ inserted in the acquisition function. Indeed, these outliers could have created a bias during the training of $\hat h$, by getting an ill-deserved representation in the weights of the estimator ( remember that we do not reset $\hat h$ parameters after we finish a round).

Besides, one can notice that WAR is always performing better than the simple disagreement sampling method. This highly suggests that the query criterion based on the Wasserstein distance that completes the acquisition function increases substantially the converging speed. This also gives an easy way to increase the performance of the model: by adding more estimators to the committee, we get a better estimation of each of their predictions' standard deviations, and the final predictions can be refined by computing the mean or the median of each of their outputs. 

\begin{table}[h]
\caption{Area under each of the models' curve}\label{tab:area}
\begin{center}
\begin{tabular}{|c|ccccc|}
\hline \
  & Bo & Ai & En & Ya & Co \\ \hline
Random & 228 & 625 & 126 & 143 & 274 \\ \hline
Disagreement & 216 & 675 & 126 & 150 & 235 \\ \hline
GSx & 212 & 660 & 121 & \textbf{126} & 233 \\ \hline
iGS & 208 & 588 & 126 & 138 & 237 \\ \hline
Euclidean & 262 & 702 & 144 & 161 & 263 \\ \hline
Cosine & 265 & 754 & 172 & 142 & 268 \\ \hline
iRDM & 210 & 584 & 121 & 140 & 229 \\ \hline
IDEAL & 211 & \textbf{558} & 119 & 135 & 224 \\ \hline
WAR & \textbf{194} & 583 & \textbf{114} & 133 & \textbf{213} \\ \hline
\end{tabular}
\end{center}
\end{table}
This gives a sense of the general performance of each method. Here, WAR is always among the best models. Thanks to this approach, few outliers are queried which leads to a bias reduction and a better performance overall in the query procedure. 

\begin{figure}[h]\label{graphes}
\includegraphics[width=1\linewidth, height=5cm]{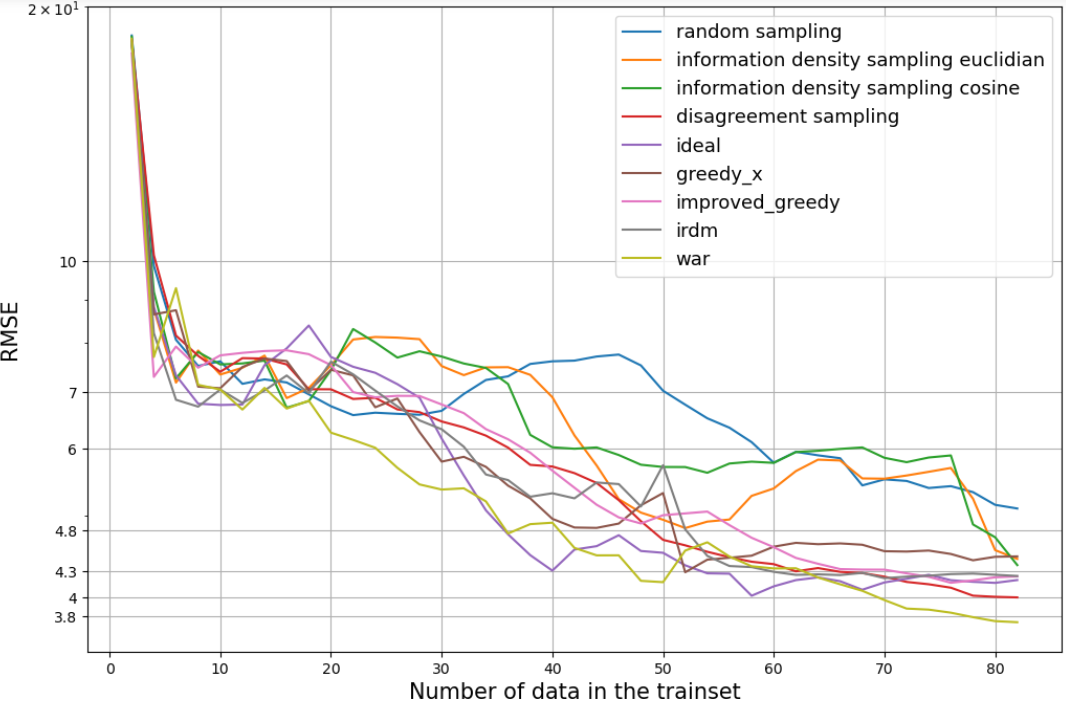}
\caption{ Evolution of the mean of each model RMSE when increasing the trainset size for Concrete Slump data set}\label{courb:concrete_slump}
\end{figure}

\subsection{A Link to curriculum learning}\label{sec:curriculum}

The formula $S(x)$ that we use to decide the next points to query can be seen as a way to quantify the amount of information provided by each data in $U$ to achieve the given task (approximating $h$). So it seems reasonable to give more weight to the points carrying the most information. In our case, the data that were queried during the first rounds contain arguably more information than the ones queried after. By not resetting $\hat h$ parameters after each round, the model will be trained more often on the points queried at the beginning. This approach can be linked to curriculum learning introduced by \cite{BLCW2009}, where available data are ranked from easy to hard and then train on them in that order. In active learning, we are not explicitly trying to rank data in this manner. However, the two fields share similar approaches to reach a good estimation of $h$ since the model is trained on a richer and more complex dataset after each iteration. We refer to the survey \cite{SIRS2022} for more information about curriculum learning.

\section{Conclusion and future works}\label{sec:conclu}

We proposed an active learning strategy for regression using probability distribution matching and uncertainty-based sampling and provided theoretical foundations. The model relies on three important aspects of data distributions in general (namely: uncertainty, diversity, and representativity) to evaluate the benefit of each point regarding our problem. This makes it more adjustable to every input space in a fashion that reduces the risk of querying outliers. The use of GroupSort neural networks has shown its relevance in such a setup thanks to their theoretical guarantees and good convergence properties. The empirical study highlights the efficiency of combining uncertainty-based approaches with the Wasserstein distance to select data and shows that keeping the weights of the estimator after each round can lead to better results. A strength of this methodology is its versatility. In fact, it can be applied with any learner $\hat h$ as long as this learner is Lipschitz continuous. 

For future works, we plan to study the impact of the query batch size ($n_B$) on the model convergence to improve our query strategy. We will also develop methods to estimate the best values of the hyperparameters $\beta$, and $\alpha$, according to the other parameters of the problem. Finally, we wish  to link our method to other curriculum learning approaches and see if some ideas of this domain could be extended to active learning.

\section{Additional experiments}

\subsection{Benchmarks details}

This section summarizes some details on the dataset and models used in the benchmark study. The table \ref{tab:hyppar} gives the values of the different models. Whereas the table \ref{tab:datsize} provides the size information about datasets, such as numbers features, entries or different subsamples size. 
\begin{figure}[h]
\begin{center}
\begin{tabular}{|c|ccccc|}
\hline
 & Boston & Airfoil & Energy1 & Yacht & Concrete \\ \hline
WAR - $\alpha$ & 2 & 0 & 0 & 1 & 1 \\ \hline
WAR - $\beta$ & 3 & 2.5 & 2.5 & 3 & 6 \\ \hline
iRDM - $c_{max}$  & 5 & 5 & 5 & 5 & 5 \\ \hline
IDEAL - $\omega$ & 3 & 0.5 & 1 & 1 & 1 \\ \hline
IDEAL - $\delta$ & 3 & 3 & 3 & 0.3 & 0.1 \\ \hline
\end{tabular}
\caption{Model hyperparameters}\label{tab:hyppar}
\end{center}
\end{figure}

The signification of other models' hyperparameters:
\begin{itemize}
\item$c_{max}$ = maximum number of iteration
\item$\delta$ = weight of the exploration factor
\item$\omega$= weight of the density measure
\end{itemize}

\begin{figure}[h]
\begin{center}
\begin{tabular}{|c|ccccc|}
\hline
 & Boston & Airfoil & Energy1 & Yacht & Concrete \\ \hline
number of features & 13 & 5 & 8 & 6 & 7 \\ \hline
trainset length & 404 & 1202 & 614 & 246 & 82 \\ \hline
testset length & 102 & 301 & 154 & 62 & 21 \\ \hline
initial labeled pool length & 8 & 24 & 12 & 5 & 2 \\ \hline
batch queried length & 8 & 24 & 12 & 5 & 2 \\ \hline

\end{tabular}
\caption{Dataset details}\label{tab:datsize}
\end{center}
\end{figure}

\subsection{Graphs}
The non-displayed learning curves in the core document (mentioned in section 4 can be found in this section. 
All figures show the same behavior as described in the core document, which comforts the mentioned conclusions.


\begin{figure}[h]
\begin{subfigure}{0.5\textwidth}
\includegraphics[width=0.9\linewidth, height=4cm]{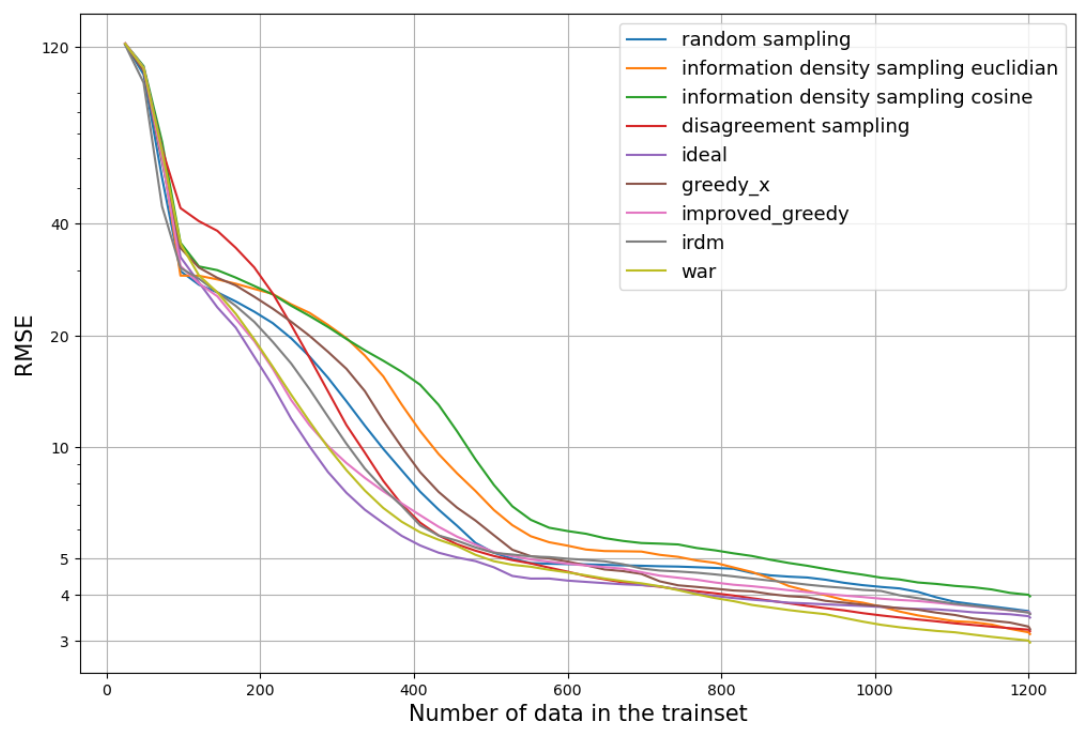} 
\caption{Airfoil Self-Noise}\label{courb:airfoil}
\end{subfigure}
\begin{subfigure}{0.5\textwidth}
\includegraphics[width=0.9\linewidth, height=4cm]{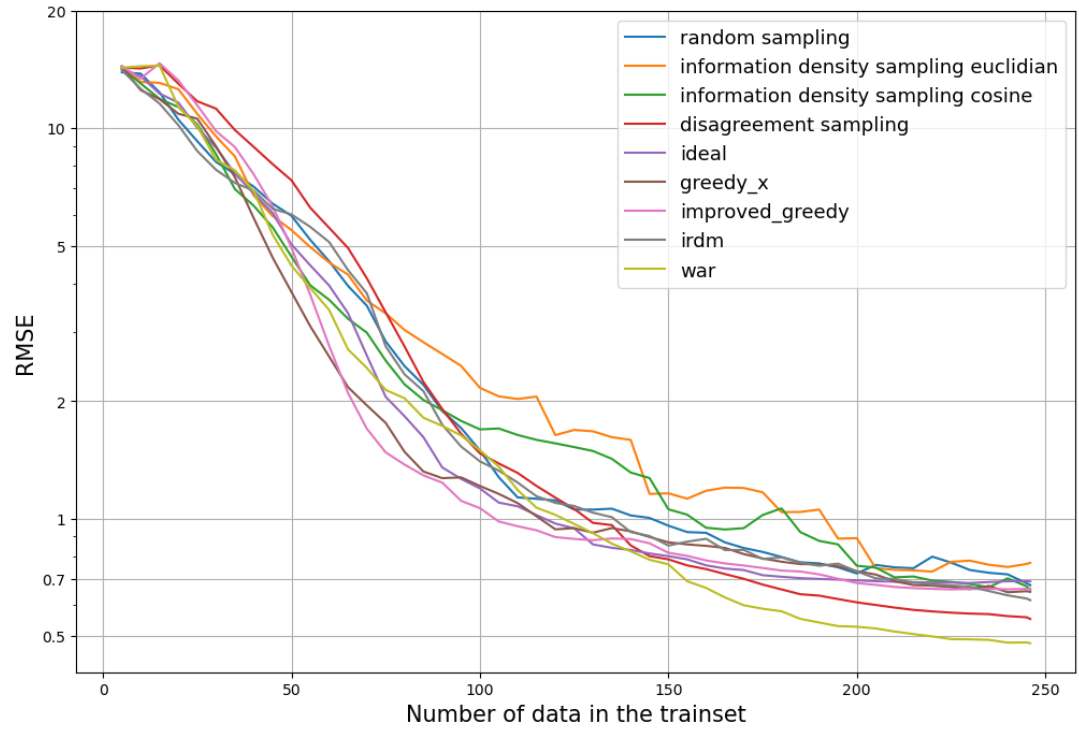}
\caption{Yacht Hydrodynamics}\label{courb:yacht}
\end{subfigure}
\begin{subfigure}{0.5\textwidth}
\includegraphics[width=0.9\linewidth, height=4cm]{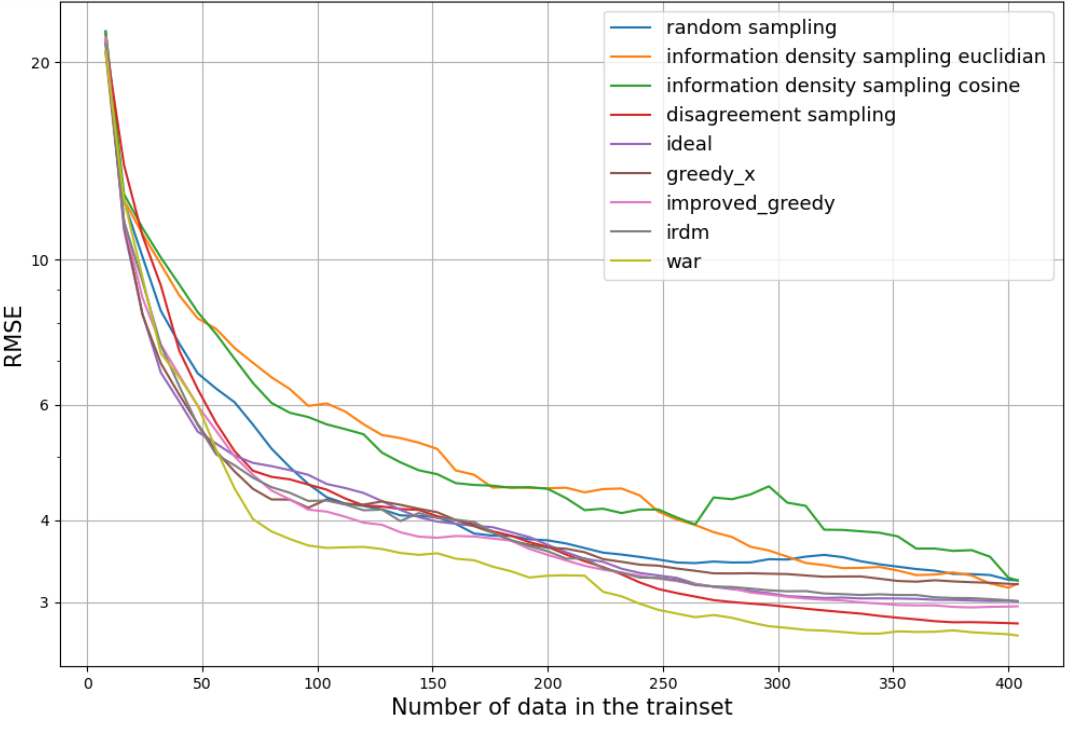} 
\caption{Boston Housing}\label{courb:airfoil}
\end{subfigure}
\begin{subfigure}{0.5\textwidth}
\includegraphics[width=0.9\linewidth, height=4cm]{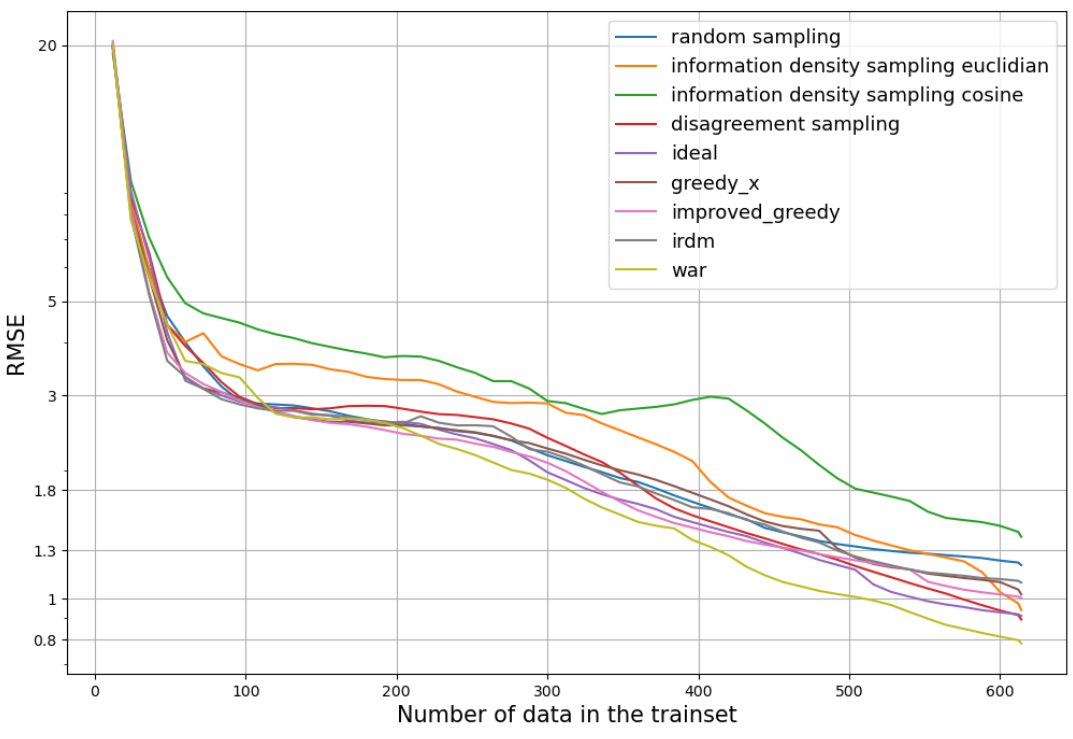}
\caption{Energy Efficiency}\label{courb:yacht}
\end{subfigure}
\caption{Evolution of the mean of each model RMSE when increasing
the trainset size}
\label{fig:image2}
\end{figure}
\newpage

\bibliography{biblio.bib}
\bibliographystyle{apalike}

\end{document}